\documentclass[twoside]{article}

\usepackage[preprint]{aistats2027}

\usepackage[round]{natbib}

\usepackage{amsmath,amssymb,amsfonts,amsthm,bm}
\usepackage{graphicx}
\usepackage{booktabs}
\usepackage{multirow}
\usepackage{url}
\usepackage{xcolor}
\usepackage{hyperref}

\definecolor{darkblue}{rgb}{0, 0, 0.5}
\hypersetup{colorlinks=true, citecolor=darkblue, linkcolor=darkblue, urlcolor=darkblue,
  pdftitle={Closing the Loop: Koopman Observers for Diffusion Acceleration}}

\newtheorem{proposition}{Proposition}

\begin{document}

%
\runningtitle{Koopman Observers for Diffusion Acceleration}

%

\twocolumn[

\aistatstitle{Koopman Observers for Diffusion Acceleration: Correcting Feature Forecasts with Shallow Measurements}

\aistatsauthor{Hanru Bai \And Yuanchao Xu \And Fengyi Li }

\aistatsaddress{ ETH Zurich \\ Max Planck Institute for Intelligent Systems \And Kyoto University \And MIT } ]

\begin{abstract}
Feature caching accelerates diffusion sampling by replacing expensive network evaluations with predictions from previously computed activations. However, forecasts based only on past features cannot directly incorporate changes in the current denoising state. We investigate whether inexpensive, freshly computed features can serve as observations for correcting these predictions. We introduce an observation-corrected Koopman framework for accelerating frozen diffusion models. Using calibration trajectories, we identify finite-dimensional, time-dependent Koopman approximations that jointly describe the increments of shallow and deep network features. During accelerated sampling, these operators predict the evolution of expensive deep features, while innovations in the observed shallow features correct the predicted state. Periodic full evaluations refresh the observer, and all generative-model parameters remain unchanged. This formulation enables controlled comparisons of temporal
prediction and observation correction.
Across three 10,000-image runs per dataset, our method reduces
paired Inception-feature MSE by $19.9\%$ on CIFAR-10 and
$11.9\%$ on a ten-class ImageNet subset relative to channelwise
affine prediction under the same four-partial-step schedule.
Matched ablations attribute additional reductions of
$4.54\%$ and $4.67\%$ to observation correction.
The observer achieves $1.89\times$ and $1.85\times$ measured
speedups over DDIM-50, supporting improved reference-sampler
fidelity without retraining the denoiser.
\end{abstract}

\section{Introduction}
\label{sec:introduction}

Diffusion models generate samples through repeated denoising network evaluations, making inference computationally expensive~\citep{ho2020denoising}. Feature caching reduces this cost by reusing intermediate activations, while forecasting methods predict their evolution between full evaluations~\citep{ma2024deepcache}. Approaches such as TaylorSeer~\citep{liu2025reusing} and FastCache~\citep{liu2025fastcache} demonstrate the potential of temporal extrapolation and learned linear prediction. However, improving the forecaster alone leaves an important question: how can predictions incorporate information about the current denoising state?

A forecast constructed exclusively from past activations cannot
directly incorporate changes in the current denoising state.
This matters because approximation errors alter subsequent network
inputs, potentially making historical trajectories less informative.
Nevertheless, a partial forward pass can provide inexpensive,
fresh shallow features even when deeper computation is skipped.
DeepCache~\citep{ma2024deepcache} already exploits this computational structure;
our question is whether these features can additionally serve as
\emph{measurements for correcting deep-feature forecasts}.
This differs from controlling forecast reliability:
RACER uses disagreement between forecasts to adjust trust and refresh decisions~\citep{li2026disagree}.
We instead investigate an explicit measurement innovation,
formed from the discrepancy between predicted and observed shallow features, to update the unobserved deep state. Figure~\ref{fig:intro_imagenet_examples} illustrates selected cases
in which observation correction improves agreement with full-model
generation.

\begin{figure}[!t]
    \centering
    \includegraphics[width=\linewidth]
    {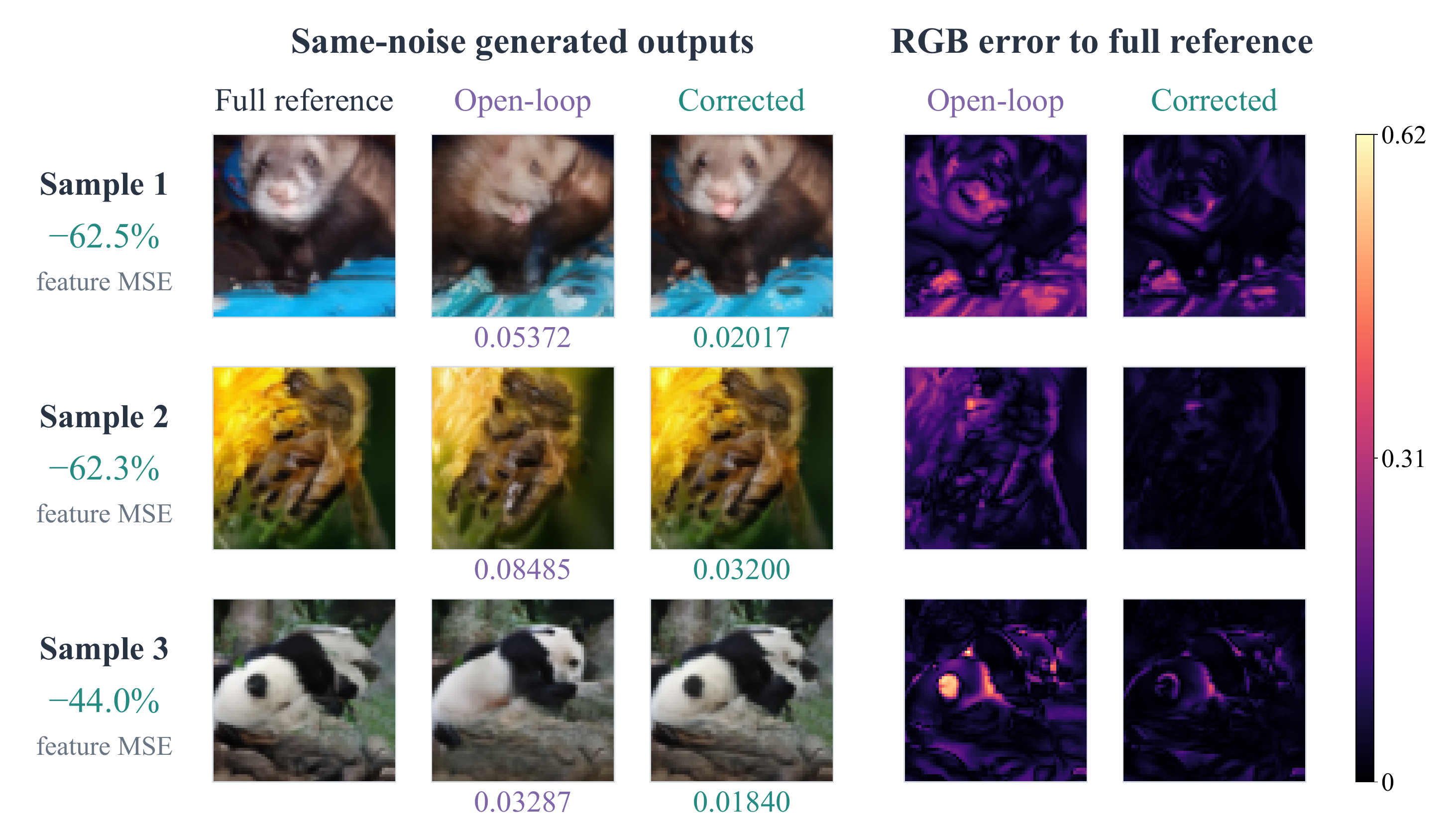}

    \caption{\textbf{Observation correction on selected ImageNet examples.}
    Each row compares full-model generation, open-loop prediction,
    and observation-corrected prediction using the same initial
    noise and class label. The accelerated variants share the
    full-evaluation schedule.
    Examples are selected for large pixel-MSE reductions among
    samples that also improve in Inception-feature MSE, with one
    example per class. They illustrate improvement cases, not
    average performance.
    Images are generated at $64\times64$ on a ten-class ImageNet
    subset. Full-model outputs are generated references, not
    ground-truth images.
    Heatmaps show per-pixel RGB RMSE on $[0,1]$-scaled images
    using a shared color scale. Numerical values and percentage
    reductions report Inception-feature MSE relative to the
    full-model reference.}
    \label{fig:intro_imagenet_examples}
\end{figure}

We introduce an \emph{observation-corrected Koopman framework} for accelerating frozen diffusion models.
Koopman methods describe nonlinear dynamics through the evolution
of observable functions, with finite-dimensional approximations
identified from data~\citep{williams2015data,brunton2021modern}.
Building on this perspective, we construct joint observables from
shallow- and deep-feature increments and their recent history.
The frozen network supplies nonlinear feature representations,
while time-dependent affine operators capture their temporal
evolution and cross-feature coupling.
These are calibrated approximations: we do not assume that the
chosen observables form an exact Koopman-invariant subspace.

Following the prediction--correction principle of state
observers~\citep{luenberger1966observers}, each accelerated step
first forecasts the joint feature increment.
A partial network evaluation then supplies the current shallow
observation.
A calibrated gain maps its innovation into a correction of the
deep-feature estimate, which is passed to the original decoder.
Periodic full evaluations refresh the observer.
This closes the prediction loop through measurements of the
current sampling state, without accessing the skipped deep
features.
All denoiser parameters remain frozen; only feature bases,
transition operators, and correction gains are fitted offline.
Our contributions are threefold:
\begin{itemize}
    \item \textbf{Joint Koopman modeling of diffusion features.}
    We formulate feature acceleration through time-dependent,
    finite-dimensional operators on joint shallow--deep
    increment observables.
    This representation makes temporal prediction and
    cross-feature coupling explicit, providing a common
    framework for distinguishing channelwise forecasting
    from coupled feature dynamics.

    \item \textbf{Measurement-driven correction without denoiser
    retraining.}
    We develop a causal sampling procedure that combines
    operator prediction, fresh shallow-feature innovations,
    and periodic full-evaluation refreshes.
    Its distinguishing mechanism is a calibrated update of
    unobserved deep features from observed prediction residuals.

    \item \textbf{Controlled evidence for the value of observation correction.}
We separate the effects of temporal prediction and feedback
using channelwise, open-loop, and observation-corrected
predictors under the same full-evaluation schedule.
Across three sampling runs with 10,000 images per dataset
per run, our observer reduces final-image Inception-feature
mean squared error by $19.9\%$ on CIFAR-10 and $11.9\%$
on a ten-class ImageNet subset relative to the channelwise
affine baseline at the four-partial-step budget.
Matched correction ablations show additional reductions
of $4.54\%$ and $4.67\%$ over disabled correction.
The evaluated observer implementations achieve
$1.89\times$ and $1.85\times$ measured speedups over
full 50-step DDIM sampling.
Stage-wise diagnostics and information-matched controls
further reveal when correction helps and characterize
the limitations of the observer parameterization.
\end{itemize}

\section{Related Work}
\label{sec:related_work}

\paragraph{Diffusion acceleration and feature forecasting.}
Diffusion models require repeated denoising
evaluations~\citep{ho2020denoising}.
DDIM and DPM-Solver reduce the number of sampling
evaluations~\citep{song2020denoising,lu2022dpm},
whereas feature caching reduces computation within each step.
DeepCache updates inexpensive shallow features while reusing
deep representations~\citep{ma2024deepcache}.
TaylorSeer, Spectrum, and L2P forecast features using temporal
extrapolation, spectral expansions, and learned historical
coefficients, respectively
~\citep{liu2025reusing,han2026adaptive,shen2026beyond}.
FastCache further incorporates learnable linear
approximations~\citep{liu2025fastcache}.
Our work investigates how freshly computed shallow
observations can correct deep-feature predictions.

\paragraph{Adaptive caching and feedback.}
HarmoniCa addresses trajectory dependencies and
training--inference mismatch in feature
caching~\citep{huang2024harmonica}.
RACER uses forecast disagreement to adjust trust and
reallocate full evaluations~\citep{li2026disagree}.
WorldDynCache combines a lifted latent surrogate with
risk-controlled exact transitions for diffusion world
models~\citep{chen2026worlddyncache}.
These methods establish adaptive control and
dynamics-inspired approximation as existing directions.
Our method computes a measurement innovation from the
discrepancy between predicted and observed shallow features.
This innovation directly updates the estimated deep state,
providing a distinct mechanism for incorporating feedback.

\paragraph{Koopman modeling and state estimation.}
Koopman methods describe nonlinear dynamics through
observable functions, with finite-dimensional approximations
identified from data
~\citep{williams2015data,brunton2021modern,bai2025konode,bai2026hierarchical}.
State observers combine predictions and measurements to
estimate unobserved variables~\citep{luenberger1966observers}.
We apply these principles to joint shallow--deep feature
increments using time-dependent affine operators and
calibrated correction gains.
Unlike WorldDynCache's local-memory surrogate, we explicitly
identify feature-transition matrices.
We assume neither an exact invariant observable subspace
nor guaranteed observer convergence.
Since an unconstrained affine observer can also be expressed
as direct regression on the same inputs, its practical value
must be assessed against same-information regression controls.

\section{Preliminaries and Problem Formulation}
\label{sec:preliminaries}

\subsection{Diffusion Sampling}

Let $\epsilon_\theta$ denote a frozen denoising network
and $\{t_k\}_{k=0}^{N-1}$ a decreasing diffusion-time grid.
The index $k$ follows the order of sampling.
For deterministic sampling, including DDIM with
$\eta=0$~\citep{song2020denoising}, the update is
\begin{equation}
    x_{k+1}
    =
    \Phi_k\bigl(x_k,\epsilon_\theta(x_k,t_k,c)\bigr),
    \label{eq:full_sampling}
\end{equation}
where $c$ denotes conditioning and $\Phi_k$ is the
prescribed sampler update.
Starting from Gaussian noise $x_0$, full sampling
evaluates the entire denoiser at every step.
Our objective is to reduce network computation while
keeping the sampling grid and denoiser parameters fixed.
For notational simplicity, guidance is omitted.
When used, feature estimation is applied separately
to the conditional and unconditional branches.

\subsection{Partial Evaluation and Feature Estimation}

We partition the denoiser into a shallow computation,
an expensive deep computation, and a final readout.
At the current accelerated sampling state $\tilde{x}_k$,
define
\begin{equation}
    \begin{aligned}
        s_k &= S_\theta(\tilde{x}_k,t_k,c), \\
        d_k &= D_\theta(\tilde{x}_k,t_k,c).
    \end{aligned}
    \label{eq:feature_definitions}
\end{equation}
Here $s_k$ contains the shallow features required by
the retained skip connections, and $d_k$ is the deep
representation entering the final decoder.
A full evaluation produces
$R_\theta(d_k,s_k,t_k,c)$.
The functions $S_\theta$, $D_\theta$, and $R_\theta$
are components of the original frozen network.

Let $\mathcal{F}$ denote the full-evaluation indices.
Shallow features are computed at every step, while
deep features are evaluated only when $k\in\mathcal{F}$.
Otherwise, a causal estimator supplies
\begin{equation}
    \hat{d}_k
    =
    \begin{cases}
        d_k, & k\in\mathcal{F}, \\
        G_k(\mathcal{H}_{k-1},s_k), & k\notin\mathcal{F},
    \end{cases}
    \label{eq:feature_estimation}
\end{equation}
where $\mathcal{H}_{k-1}$ contains previously computed
observations, cached features, and estimated states.
The estimator cannot access the current deep feature
at a skipped evaluation.
The accelerated update becomes
\begin{equation}
    \tilde{x}_{k+1}
    =
    \Phi_k\bigl(
        \tilde{x}_k,
        R_\theta(\hat{d}_k,s_k,t_k,c)
    \bigr).
    \label{eq:accelerated_sampling}
\end{equation}
Importantly, $d_k$ is defined at $\tilde{x}_k$.
It can differ from the deep feature along the
full-sampling reference trajectory because earlier
approximations change subsequent inputs.

\subsection{Finite-Dimensional Koopman Approximation}

For dynamics $q_{k+1}=F_k(q_k)$, the time-dependent
Koopman operator acts on an observable function $g$ as
\begin{equation}
    (\mathcal{K}_k g)(q)
    =
    g(F_k(q)).
    \label{eq:koopman_operator}
\end{equation}
This operator is linear in $g$, even when $F_k$
is nonlinear~\citep{brunton2021modern, 10.1063/5.0283640}.
The state $q_k$ may include the sample, conditioning,
time, and history required to define the observables.

Given a finite observable dictionary
$v_k=\psi(q_k)$, data-driven identification yields
an affine approximation
\begin{equation}
    v_{k+1}
    =
    A_k v_k+b_k+r_k,
    \label{eq:finite_koopman}
\end{equation}
where $r_k$ captures finite-dictionary closure and
identification errors~\citep{williams2015data}.
We do not assume that this residual vanishes.

In the next section, we construct $v_k$ from joint
shallow--deep feature increments and use current
shallow observations to correct its predicted evolution.
The observable bases, transition matrices, and correction
gains are calibrated offline without updating $\theta$.
We evaluate approximation fidelity and sampling cost
separately from distributional generation quality.


\section{Method}
\label{sec:method}

We construct a feature observer with three operations: joint increment
prediction, shallow-observation correction, and periodic full refreshes.
Figure~\ref{fig:koopman_observer_combined} (left) summarizes the partial
evaluation path, the observer update, and offline calibration.
All denoiser parameters remain frozen.

\begin{figure*}[!t]
    \centering
    \begin{minipage}[t]{0.49\linewidth}
        \vspace{0pt}
        \centering
        \includegraphics[width=\linewidth]
        {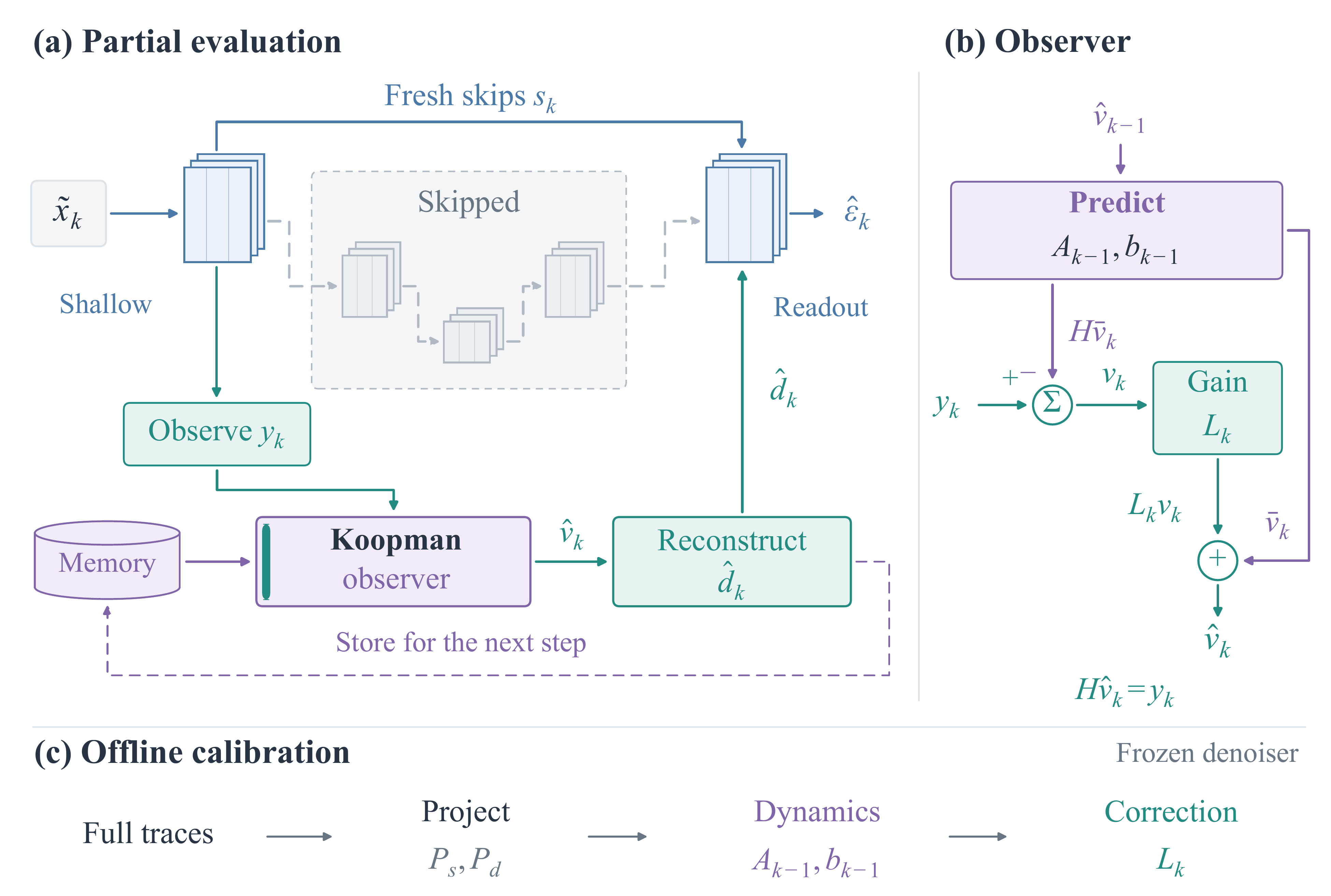}
    \end{minipage}\hfill
    \begin{minipage}[t]{0.49\linewidth}
        \vspace{0pt}
        \centering
        \includegraphics[width=\linewidth]
        {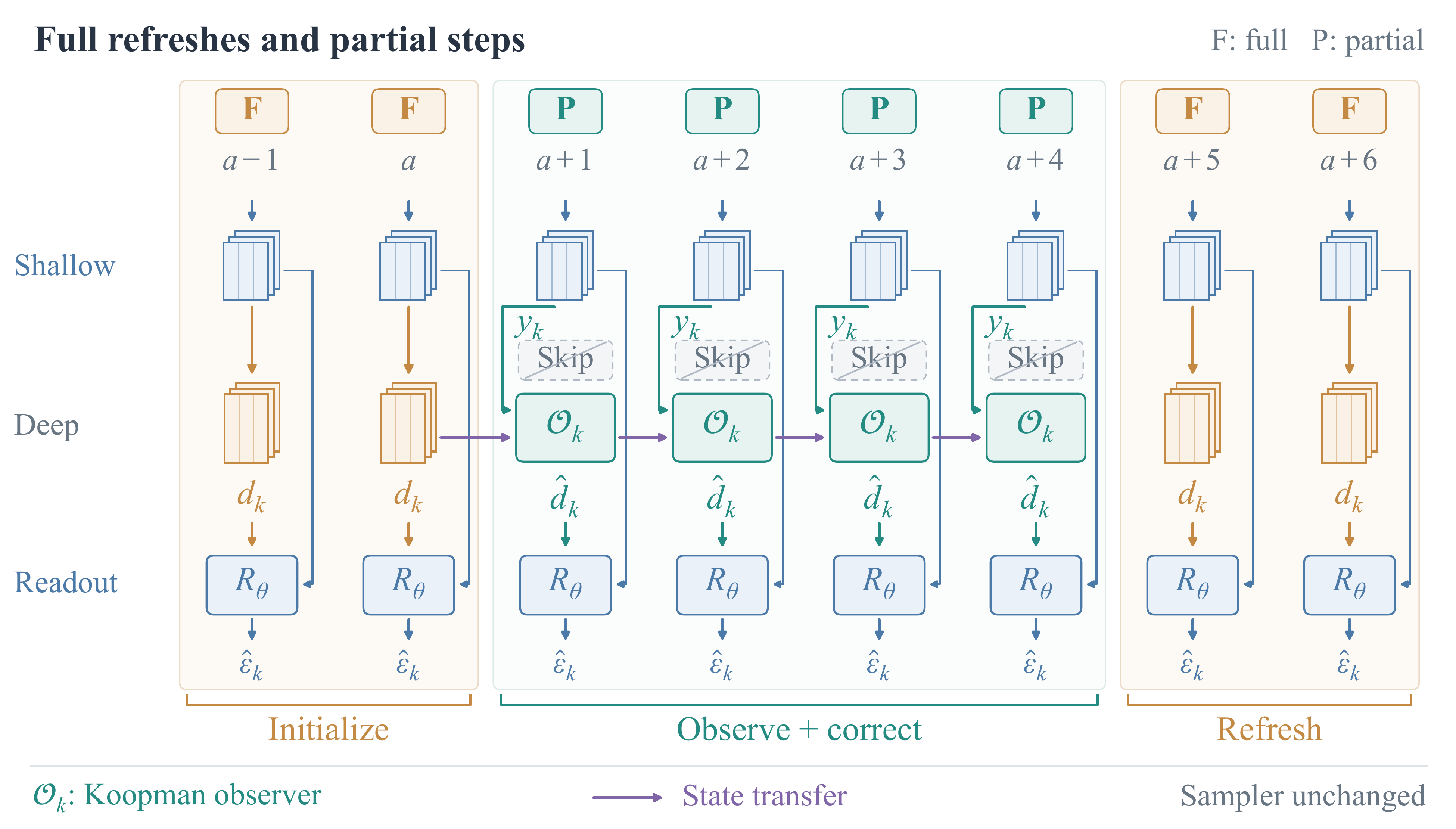}
    \end{minipage}
    \caption{\textbf{Koopman observers for diffusion acceleration.}
    \textbf{Left:} Fresh shallow observations correct the predicted joint
    increment state. The reconstructed deep feature and current shallow
    skips feed the retained readout, while expensive deep blocks are
    skipped. Projection bases, scales, time-dependent operators, and
    correction gains are identified from offline calibration trajectories.
    \textbf{Right:} Two full evaluations (F) initialize the observer state.
    Partial evaluations (P) use the Koopman observer $\mathcal{O}_k$ to
    predict, correct, and reconstruct deep features. Subsequent full
    evaluations refresh the estimates at the current accelerated samples.
    The schedule shows an illustrative interior segment.
    All denoiser weights remain frozen and the sampler is unchanged.}
    \label{fig:koopman_observer_combined}
\end{figure*}

\subsection{Joint Increment Observables}
\label{sec:method_observables}

We describe the channel vectors at one spatial location and suppress the
conditioning branch. Features are spatially aligned, and the same operators
are applied at every location. Let $P_d\in\mathbb{R}^{C_d\times r}$ and
$P_s\in\mathbb{R}^{C_s\times r}$ be orthonormal channel bases obtained from
the second moments of calibration feature increments. With positive
calibration scales $\sigma_d$ and $\sigma_s$, define
\begin{equation}
    u_k=\frac{P_d^\top(d_k-d_{k-1})}{\sigma_d},
    \;
    y_k=\frac{P_s^\top(s_k-s_{k-1})}{\sigma_s},
    \;
v_k=\begin{bmatrix}u_k\\y_k\end{bmatrix}.
    \label{eq:joint_increments}
\end{equation}
The scales are the root mean square values of the projected increments.
Let $J=[I_r\ \ 0]$ and $H=[0\ \ I_r]$ select the deep and shallow components.
The shallow increment $y_k$ remains observable at partial evaluations.
Measuring the deep increment $u_k$ requires consecutive full evaluations.
These increments are nonlinear observables of the sample and its preceding
state because the underlying features come from the frozen network.

\subsection{Offline Operator and Gain Identification}
\label{sec:method_identification}

Following Figure~\ref{fig:koopman_observer_combined} (left, panel~c), we collect
full-sampling trajectories and split complete trajectories into
calibration, validation, and held-out diagnostic sets. Patches and
conditioning branches from one trajectory remain in the same split.
For each transition, affine
EDMD-style regression~\citep{williams2015data} gives
\begin{equation}
    \begin{aligned}
    (A_{k-1},b_{k-1})=\arg\min_{A,b}\quad&
    \frac{1}{M_k}\sum_{j=1}^{M_k}
    \left\|v_k^{(j)}-Av_{k-1}^{(j)}-b\right\|_2^2\\
    &+\lambda_k\left(\|A-I_{2r}\|_F^2+\|b\|_2^2\right),
    \end{aligned}
    \label{eq:operator_identification}
\end{equation}
where $j$ indexes calibration feature vectors. The identity prior favors
persistence of feature increments. The indexing follows the transition
$v_{k-1}\mapsto v_k$ defined in the preliminaries.

Let $\rho_k^{(j)}=v_k^{(j)}-A_{k-1}v_{k-1}^{(j)}-b_{k-1}$ be the fitted
residual. We learn a shallow-to-deep residual gain through
\begin{equation}
    \begin{aligned}
    L_k^d&=\arg\min_L\
    \frac{1}{M_k}\sum_{j=1}^{M_k}
    \left\|J\rho_k^{(j)}-LH\rho_k^{(j)}\right\|_2^2
    +\gamma_k\|L\|_F^2,\\
    L_k&=\begin{bmatrix}L_k^d\\I_r\end{bmatrix}.
    \end{aligned}
    \label{eq:gain_identification}
\end{equation}
The effective ridge penalties include input second-moment scaling.
Bases, scales, operators, and gains use calibration data only.
Rank and regularization are selected on validation trajectories before
evaluation on independent sampling cohorts.

\subsection{Observation-Corrected Prediction}
\label{sec:method_correction}

As illustrated in Figure~\ref{fig:koopman_observer_combined}
(left, panels~a and~b),
at a partial evaluation, we first propagate the estimated increment state
and compute its shallow innovation
\begin{equation}
    \bar v_k=A_{k-1}\hat v_{k-1}+b_{k-1},
    \qquad
    \nu_k=y_k-H\bar v_k.
    \label{eq:observer_prediction}
\end{equation}
The prediction--correction update follows the state-observer
principle~\citep{luenberger1966observers}
\begin{equation}
    \hat v_k=\bar v_k+L_k\nu_k,
    \qquad
    \hat d_k=\hat d_{k-1}+\sigma_d P_d J\hat v_k.
    \label{eq:observer_correction}
\end{equation}
Since $HL_k=I_r$, the corrected shallow component satisfies
$H\hat v_k=y_k$. Its deep component is updated through the calibrated
cross-feature residual relation. The original decoder receives
$\hat d_k$ and the current shallow skips $s_k$, followed by the unchanged
sampler update. No current deep feature is computed during this operation.
When $r<C_d$, only the retained deep-increment subspace is updated between
refreshes.

\subsection{Refreshes and Computational Cost}
\label{sec:method_refresh}

Figure~\ref{fig:koopman_observer_combined} (right) illustrates the alternation
between full evaluations and observation-corrected partial steps.
Each prediction episode begins with two consecutive full evaluations at
indices $a-1$ and $a$. These initialize $\hat d_a=d_a$ and $\hat v_a=v_a$
from measured feature differences. Subsequent partial steps recursively
apply Eqs.~\eqref{eq:observer_prediction}--\eqref{eq:observer_correction}.
The next full-evaluation pair resets the feature and increment estimates
at the current accelerated states. It does not reset the sample to the
full-sampling reference trajectory.

For $m$ full evaluations among $N$ sampling steps, the approximate online
cost is
\begin{equation}
    T_{\mathrm{online}}
    \approx mT_{\mathrm{full}}+(N-m)T_{\mathrm{partial}},
    \label{eq:observer_cost}
\end{equation}
where $T_{\mathrm{partial}}$ includes shallow computation, projections,
prediction, correction, and the retained decoder and sampler operations.
Actual latency is measured end to end, with offline calibration reported
separately. Matching full-evaluation counts alone does not match runtime.

\subsection{Error Propagation}
\label{sec:method_error}

Let $e_k=v_k-\hat v_k$ be the estimation error and
$\rho_k=v_k-A_{k-1}v_{k-1}-b_{k-1}$ the residual of the identified
dynamics along the accelerated trajectory. Subtracting
Eqs.~\eqref{eq:observer_prediction}--\eqref{eq:observer_correction}
from Eq.~\eqref{eq:finite_koopman} and using $y_k=Hv_k$ gives
\begin{equation}
    e_k=(I_{2r}-L_kH)
    \left(A_{k-1}e_{k-1}+\rho_k\right).
    \label{eq:observer_error_identity}
\end{equation}
Because $HL_k=I_r$, the shallow component of $e_k$ vanishes,
$e_k=(e_k^u,0)$. Writing $A_{k-1}$ in blocks according to the deep and
the shallow components,
$A_{k-1}=\bigl[\begin{smallmatrix}A^{uu}_{k-1}&A^{uy}_{k-1}\\
A^{yu}_{k-1}&A^{yy}_{k-1}\end{smallmatrix}\bigr]$, the deep error
therefore obeys
\begin{equation}
    e^u_k
    =
    \bigl(A^{uu}_{k-1}-L^d_kA^{yu}_{k-1}\bigr)e^u_{k-1}
    +\bigl(\rho^u_k-L^d_k\rho^y_k\bigr).
    \label{eq:deep_error_recursion}
\end{equation}
This is the error recursion of a Luenberger observer with transition
matrix $A^{uu}_{k-1}$, output matrix $A^{yu}_{k-1}$ and gain $L^d_k$:
the measured shallow increment carries information about the
propagated deep error only through the coupling block $A^{yu}_{k-1}$.
For block-diagonal dynamics, and in particular for channelwise models,
$A^{yu}_{k-1}=0$ and the correction acts on the residual alone. We make
no claim of contraction; the closed-loop factor
$A^{uu}_{k-1}-L^d_kA^{yu}_{k-1}$ is available offline and can be
compared with $A^{uu}_{k-1}$.

The gain~\eqref{eq:gain_identification} is the ridge regression of the
deep residual on the shallow residual. Its effect on a single step is
described by the following statement, whose proof is a direct
computation (Appendix~\ref{app:proof}).

\begin{table*}[t]
    \centering
    \footnotesize
    \setlength{\tabcolsep}{4pt}
    \caption{
        Generation quality and full-sampler fidelity at the
        middle latency budget.
        Each result uses 10,000 images per sampling run and
        three runs.
        FID reports mean and sample SD.
        $E_{\phi}$ is mean paired Inception-feature MSE
        relative to full DDIM-50.
        All seven baseline families are within 3.81\% of
        the observer latency.
        $\dagger$ denotes an adaptation to the native U-Net.
    }
    \label{tab:main_results}
    \begin{tabular}{lrrr rrr}
        \toprule
        & \multicolumn{3}{c}{CIFAR-10}
        & \multicolumn{3}{c}{ImageNet-10} \\
        \cmidrule(lr){2-4}
        \cmidrule(lr){5-7}
        Method
        & FID$\downarrow$
        & $10^3E_{\phi}\downarrow$
        & Speedup$\uparrow$
        & FID$\downarrow$
        & $10^3E_{\phi}\downarrow$
        & Speedup$\uparrow$ \\
        \midrule
        Full DDIM-50
        & $13.39\pm0.05$ & 0.000 & $1.00\times$
        & $13.72\pm0.08$ & 0.000 & $1.00\times$ \\

        Reduced-step DDIM
        & $15.43\pm0.10$ & 10.419 & $1.79\times$
        & $14.55\pm0.10$ & 21.492 & $1.67\times$ \\

        DPM-Solver++ 2M
        & $16.84\pm0.10$ & 28.282 & $1.79\times$
        & $14.76\pm0.09$ & 38.720 & $1.67\times$ \\

        Feature reuse
        & $15.90\pm0.09$ & 15.985 & $1.71\times$
        & $15.95\pm0.04$ & 37.214 & $1.69\times$ \\

        Linear extrapolation
        & $13.53\pm0.07$ & 2.982 & $1.71\times$
        & $14.08\pm0.09$ & 14.924 & $1.69\times$ \\

        Channelwise affine
        & $13.59\pm0.07$ & 2.795 & $1.71\times$
        & $14.13\pm0.09$ & 14.330 & $1.68\times$ \\

        DeepCache$^{\dagger}$
        & $14.23\pm0.08$ & 5.085 & $1.67\times$
        & $14.63\pm0.09$ & 22.843 & $1.65\times$ \\

        Spectrum$^{\dagger}$
        & $15.24\pm0.09$ & 12.146 & $1.69\times$
        & $14.93\pm0.13$ & 30.712 & $1.64\times$ \\

        \midrule
        Observer ($s=3$)
        & $13.51\pm0.08$ & 1.776 & $1.73\times$
        & $14.13\pm0.10$ & 9.741 & $1.70\times$ \\
        \bottomrule
    \end{tabular}
\end{table*}

\begin{proposition}
\label{prop:onestep}
Let $\epsilon=(\epsilon^u,\epsilon^y)$ be a random vector with finite
second moments $\Sigma_{uu}=\mathbb{E}[\epsilon^u\epsilon^{u\top}]$,
$\Sigma_{uy}=\mathbb{E}[\epsilon^u\epsilon^{y\top}]=\Sigma_{yu}^\top$
and $\Sigma_{yy}=\mathbb{E}[\epsilon^y\epsilon^{y\top}]$, let
$\gamma\ge0$ be such that $\Sigma_{yy}+\gamma I$ is invertible, and let
$L=\Sigma_{uy}M$ with $M=(\Sigma_{yy}+\gamma I)^{-1}$. Then
\begin{align*}
    &\mathbb{E}\|\epsilon^u\|^2-\mathbb{E}\|\epsilon^u-L\epsilon^y\|^2\\
    &\qquad=
    \operatorname{tr}\bigl[\Sigma_{uy}M(\Sigma_{yy}+2\gamma I)M\Sigma_{yu}\bigr]
    \ \ge\ 0.
\end{align*}
For $\gamma=0$, $L$ minimizes $\mathbb{E}\|\epsilon^u-L'\epsilon^y\|^2$
over all matrices $L'$, and the second moment of $\epsilon^u-L\epsilon^y$
is the Schur complement $\Sigma_{uu}-\Sigma_{uy}\Sigma_{yy}^{-1}\Sigma_{yu}$.
\end{proposition}

Applied with $\epsilon$ distributed as the calibration residuals, the
proposition states that the correction with the
gain~\eqref{eq:gain_identification} does not increase the mean squared
deep error at the step that follows a reinitialization, where $e_a=0$
and the prior error equals $\rho_{a+1}$. At the subsequent steps the
prior error is $A_{k-1}e_{k-1}+\rho_k$, and the statement holds to the
extent that its second moments agree with those of the calibration
residuals. The ratio
$\operatorname{tr}(\Sigma_{uy}\Sigma_{yy}^{-1}\Sigma_{yu})/\operatorname{tr}\Sigma_{uu}$
is the fraction of the deep residual that is linearly predictable from
the shallow residual at step $k$; it measures how observable the deep
increment is through the shallow one.

Finally, the error of the deep feature passed to the decoder follows
from Eq.~\eqref{eq:observer_correction}. Since $P_d^\top P_d=I_r$ and
$\sigma_dP_dJv_k=P_dP_d^\top(d_k-d_{k-1})$, and since $\hat d_a=d_a$,
one obtains for $k>a$
\begin{equation}
    d_k-\hat d_k
    =
    \sigma_dP_d\sum_{i=a+1}^{k}e^u_i
    +(I-P_dP_d^\top)(d_k-d_a).
    \label{eq:deep_feature_accumulation}
\end{equation}
The first term accumulates the increment errors of the episode; the
second is the drift of the deep feature outside the retained subspace
since the last reinitialization, and it vanishes for $r=C_d$. The
number of accumulated terms grows with the length of the episode,
which is why reinitializations are kept frequent and why the
prediction error is examined as a function of the horizon
(Appendix~\ref{app:feature_prediction}).

Eqs.~\eqref{eq:observer_prediction}--\eqref{eq:observer_correction}
also show that
$\hat v_k=(I_{2r}-L_kH)(A_{k-1}\hat v_{k-1}+b_{k-1})+L_ky_k$ is an
affine map of $(\hat v_{k-1},y_k)$ whose coefficients are constrained
by the factorization through $A_{k-1}$ and $L_k$ and by the identity
prior; an unconstrained regression of the deep increment on the same
inputs is one of the controls in the experiments.

\section{Experiments}
\label{sec:experiments}

We evaluate whether Koopman observers accelerate frozen diffusion
sampling while preserving the outputs of the full sampler.
We distinguish fidelity to this sampler from distributional
generation quality, then test whether current observations
contribute beyond temporal prediction alone.

\subsection{Experimental Setup}
\label{sec:experimental_setup}

\paragraph{Models and evaluation.}
We use two frozen class-conditional U-Net checkpoints with raw
weights, on CIFAR-10 at $32\times32$ and a ten-class ImageNet
subset at $64\times64$.
The reference is deterministic DDIM-50 with classifier-free
guidance scale 2.0.
Our observer retains the same 50-step grid and uses offline
calibration to identify projection bases, dynamics, and
correction gains.
The main predictor uses rank 256 for both shallow and deep
increments.
No denoiser parameters are updated. All implementations were run on NVIDIA A100 GPUs.

Each configuration generates 10,000 images for each of three
sampling runs, with 1,000 images per class per run.
Methods share initial noise and class labels within each run.
We report FID and KID against real images, and paired
Inception-feature MSE $E_{\phi}$ against the corresponding
full-sampler outputs.
Smaller $E_{\phi}$ indicates closer agreement with the full
sampler, not necessarily better generation quality.
Summary statistics are means and sample standard deviations
across sampling runs, not independently trained models.
Real-data references, pixel errors, and calibration splits are
detailed in Appendix~\ref{app:experiments}.

\paragraph{Timing and baselines.}
We measure FP32 sampling using batch sizes of
50 for CIFAR-10 and 20 for ImageNet-10.
Latencies are medians of three synchronized repetitions after
warmup, with randomized configuration order.
Timing includes both guidance branches, feature prediction,
correction, and sampler updates, but excludes noise creation,
I/O, and metric computation.

Baselines include reduced-step DDIM~\citep{song2020denoising},
DPM-Solver++~\citep{lu2025dpmpp}, DeepCache~\citep{ma2024deepcache}, and
Spectrum~\citep{han2026adaptive}, together with feature reuse,
linear extrapolation, and channelwise affine prediction.
DeepCache and Spectrum are adaptations to our native U-Net
and feature locations.
For each observer budget, baseline configurations are selected
solely by independently measured latency from a fixed
candidate set.
We call a comparison approximately latency matched only when
the absolute latency gap is at most 5\%.

\subsection{Efficiency, Fidelity, and Generation Quality}
\label{sec:efficiency_fidelity}

We evaluate maximum partial-step runs of
$s\in\{2,3,4\}$, corresponding to 28, 23, and 20 full
steps on the 50-step grid.
Each full or partial step includes both guidance branches.
Table~\ref{tab:main_results} presents the middle budget,
where all seven baseline families meet the latency tolerance
on both datasets.
Figure~\ref{fig:speed_quality} shows all three budgets.

At the middle budget, the observer achieves $1.73\times$
acceleration on CIFAR-10 and $1.70\times$ on ImageNet-10.
Its FID is $13.515\pm0.079$ and $14.133\pm0.097$,
compared with $13.393\pm0.049$ and $13.724\pm0.083$
for full DDIM-50.
The corresponding increases in FID are 0.122 and 0.409.
At approximately matched latency, the observer reduces
$E_{\phi}$ by 40.5\% and 34.7\% relative to linear
extrapolation, and by 36.5\% and 32.0\% relative to
channelwise affine prediction.
Its mean FID is also lower than those of reduced-step DDIM,
DPM-Solver++, feature reuse, and the adapted DeepCache
and Spectrum configurations at this budget.

\begin{figure}[t]
    \centering
    \includegraphics[width=\linewidth]
        {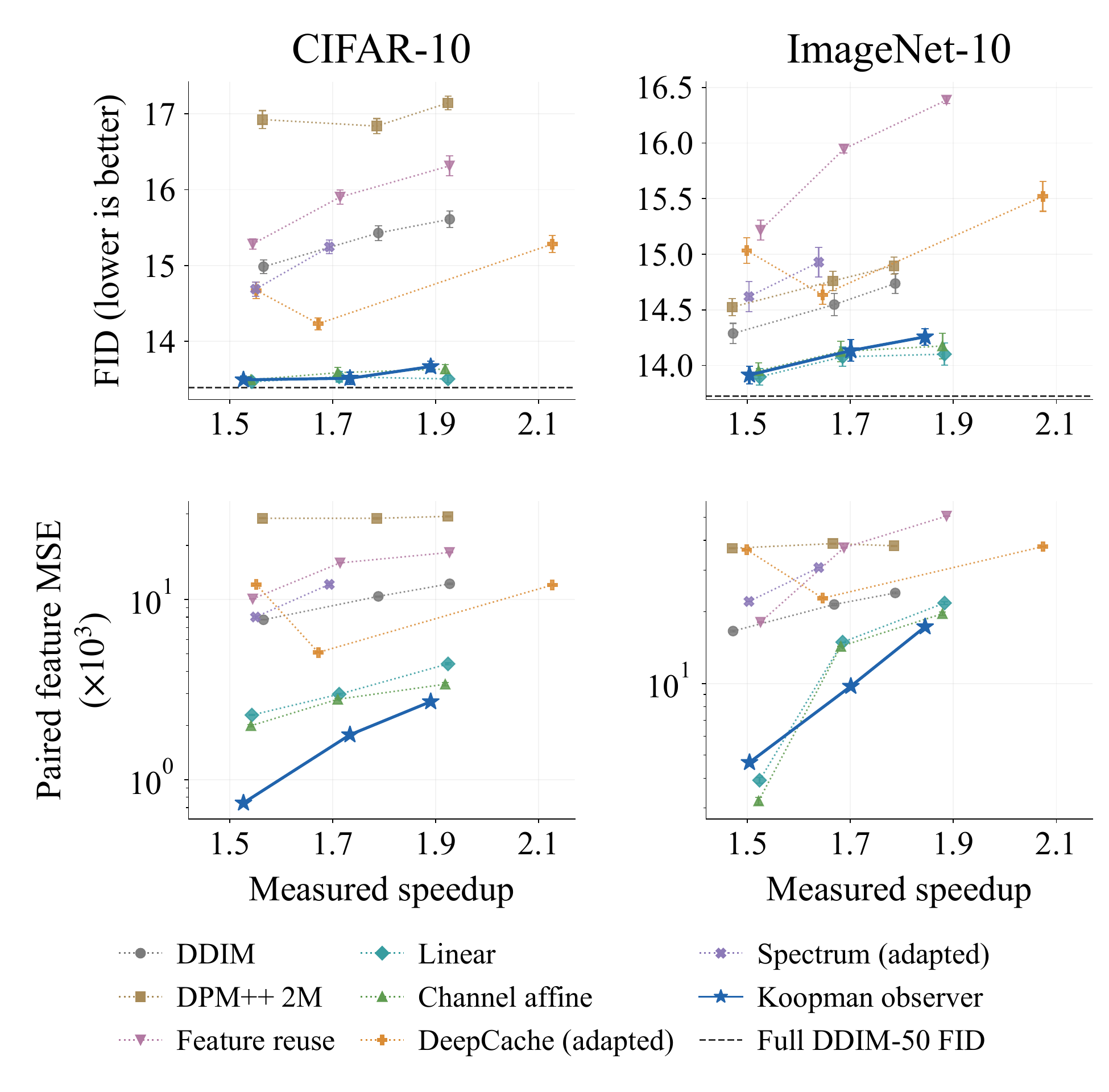}
    \caption{
        \textbf{Distributional quality and reference fidelity
        are different objectives.}
        FID (top) and paired Inception-feature MSE (bottom)
        against measured speedup on CIFAR-10 (left) and
        ImageNet-10 (right).
        Points show means over three sampling runs, with
        error bars indicating sample SD.
        The bottom row uses a logarithmic vertical axis.
        Baselines are the configurations selected by timing
        for the three observer budgets, with repeated
        selections shown once.
        Every point is plotted at its actual measured
        speedup, including configurations outside the
        5\% matching tolerance.
        Lines connect evaluated configurations as visual
        guides.
        Dashed horizontal lines show full DDIM-50 FID.
    }
    \label{fig:speed_quality}
\end{figure}



\subsection{Do Current Observations Improve Prediction?}
\label{sec:observation_ablation}

To isolate measurement correction, we fix the rank-256
predictor, projection bases, dynamics, gains, and $s=4$
refresh positions.
We compare current shallow increments, increments delayed
by one sampling step, and correction disabled.
All three variants use the same unfused implementation
and recompute the current shallow skip features supplied
to the decoder.
Only the correction input or its application changes.

\begin{table*}[t]
    \centering
    \footnotesize
    \setlength{\tabcolsep}{3pt}
    \caption{
        Effect of current observations with a fixed fitted
        predictor and the same $s=4$ schedule.
        All three variants use the unfused implementation
        and fresh shallow decoder skips.
        $E_x$ is pixel MSE on floating-point images
        in $[-1,1]$.
        Errors and FID report mean and sample SD across
        three sampling runs.
        Latency is the median of three synchronized
        repetitions.
    }
    \label{tab:observation_ablation}
    \begin{tabular}{llrrrr}
        \toprule
        Dataset
        & Correction input
        & FID$\downarrow$
        & $10^3E_{\phi}\downarrow$
        & $10^3E_x\downarrow$
        & s/batch \\
        \midrule
        CIFAR-10
        & Disabled
        & $13.638\pm0.070$
        & $2.842\pm0.050$
        & $0.244\pm0.003$
        & 3.757 \\

        CIFAR-10
        & One-step delayed
        & $13.807\pm0.067$
        & $3.110\pm0.027$
        & $0.308\pm0.003$
        & 3.787 \\

        CIFAR-10
        & Current
        & $13.669\pm0.071$
        & $2.713\pm0.047$
        & $0.224\pm0.005$
        & 3.781 \\

        \midrule
        ImageNet-10
        & Disabled
        & $14.266\pm0.107$
        & $18.172\pm0.230$
        & $4.277\pm0.052$
        & 5.310 \\

        ImageNet-10
        & One-step delayed
        & $14.335\pm0.021$
        & $20.984\pm0.124$
        & $5.086\pm0.009$
        & 5.357 \\

        ImageNet-10
        & Current
        & $14.257\pm0.073$
        & $17.323\pm0.235$
        & $3.900\pm0.045$
        & 5.344 \\
        \bottomrule
    \end{tabular}
\end{table*}

Current observations reduce $E_{\phi}$ by 4.54\%
on CIFAR-10 and 4.67\% on ImageNet-10 relative
to disabled correction.
Relative to delayed observations, the reductions are
12.78\% and 17.44\%.
For both contrasts, the per-run 95\% paired-bootstrap
intervals for the error ratio lie below one in all
three runs on both datasets
(Appendix~\ref{app:paired_metrics}).
Current observations also reduce class-conditional mean
error relative to disabled correction in all ten classes
of each dataset after averaging the three runs.

The measured latency overhead over disabled correction
is 0.63\% on CIFAR-10 and 0.66\% on ImageNet-10
in the matched unfused implementation.
Thus, the benefit of fresh measurements is
a reproducible reduction in full-sampler approximation
error.
The delayed-observation intervention tests this fixed
observer's dependence on timely inputs.
It does not establish that a predictor fitted specifically
for delayed measurements would perform equally poorly.

\subsection{Dependence on Sampling Stage and Prediction Horizon}
\label{sec:stage_horizon}

We reuse the reference-trajectory diagnostics to examine where
observation correction is effective.
Figure~\ref{fig:stage_horizon} evaluates eight anchor steps and
prediction horizons of one, two, and four steps.
Only positions that are partial evaluations under the $s=4$
schedule are included.

Correction reduces local deep-feature error at all 23 evaluated
positions on CIFAR-10, but at only 7 of 23 positions on ImageNet-10.
At the earliest ImageNet-10 anchor, the reductions range from
13.6\% to 19.1\%, while most middle and late positions show small
increases in error.
These results reveal stage-dependent benefits of correction.
They describe prediction on reference trajectories and do not
include feedback through the evolving accelerated image state.

\begin{figure}[t]
    \centering
    \includegraphics[width=\linewidth]
        {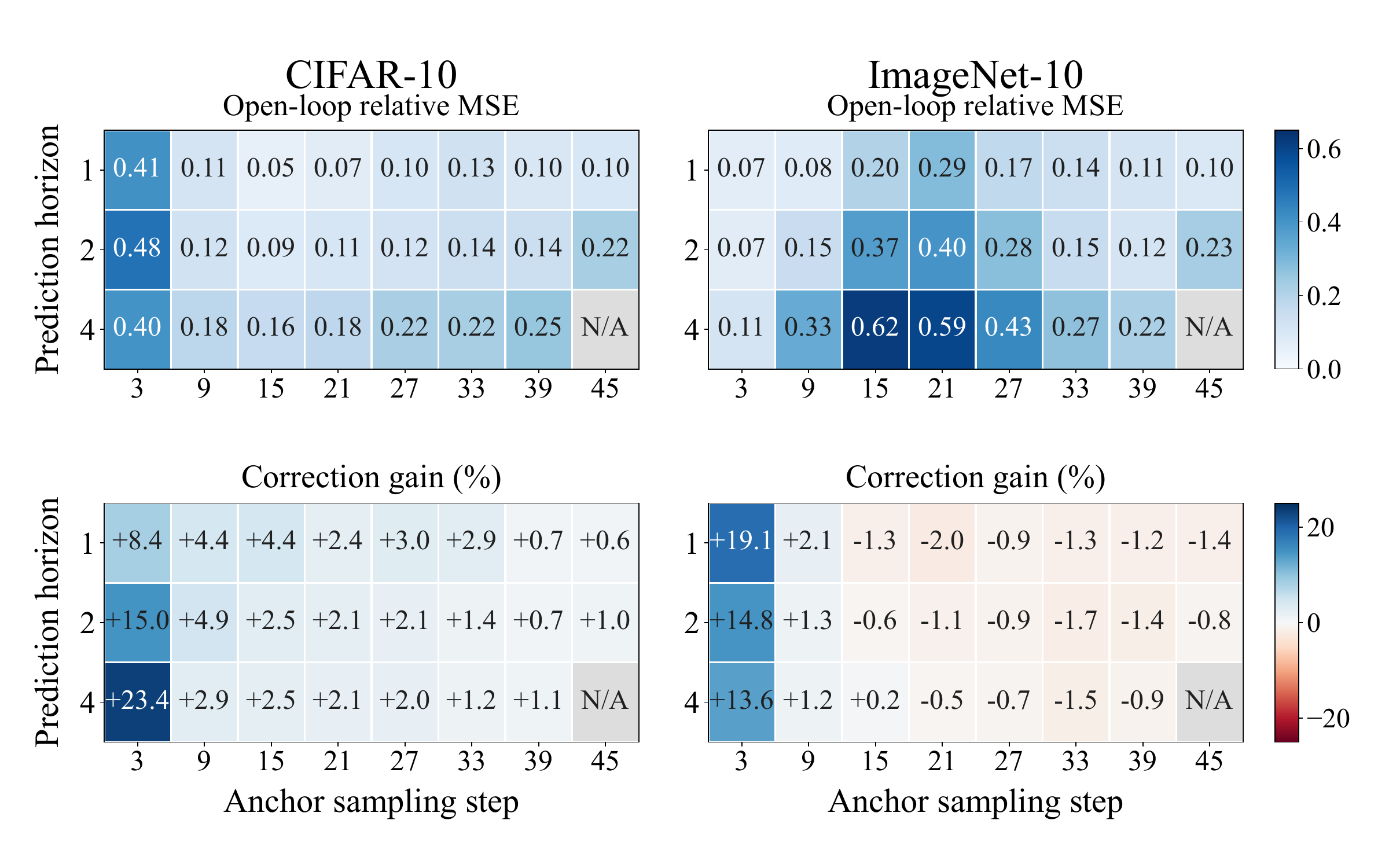}
    \caption{
    \textbf{Stage-dependent prediction and correction.}
    Columns show CIFAR-10 and ImageNet-10.
    Top: open-loop deep-feature squared error normalized by
    feature-reuse error at the same position.
    Bottom: correction gain,
    $100(1-E_{\mathrm{obs}}/E_{\mathrm{open}})$.
    Positive values indicate improvement.
    Each cell aggregates 60 CIFAR-10 or 40 ImageNet-10
    development trajectories.
    Anchor steps are zero based and follow sampling order
    from high to low noise.
    Gray cells target full evaluations and are excluded.
    }
    \label{fig:stage_horizon}
\end{figure}

\subsection{The Role of Temporal Dynamics}
\label{sec:dynamics_ablation}

Additional controls distinguish temporal modeling from
static shallow-to-deep regression and assess the effects
of state dimension and cross-channel coupling
(Table~\ref{tab:dynamics_controls}).
On the same $s=4$ schedule, static regression gives
$E_{\phi}=0.012275$ on CIFAR-10 and 0.034979
on ImageNet-10, compared with 0.002713 and 0.017323
for the rank-256 observer.

Increasing the observer rank from 64 to 256 reduces
paired feature error by 10.7\% and 11.0\%.
At rank 64, dense dynamics improve on diagonal dynamics
by 12.6\% on CIFAR-10 but only 0.5\% on ImageNet-10.
These configuration comparisons support temporal state
modeling in the evaluated predictor family, while showing
that the benefit of cross-channel coupling is
dataset dependent.

The earlier full-trajectory diagnostics show the same
aggregate ordering of linear extrapolation, fitted
open-loop dynamics, and observation-corrected dynamics
(Appendix~\ref{app:feature_prediction}).
The final-image evaluations additionally account for
changes to subsequent sampling states caused by earlier
approximations.
Batch-size timing and the limits of the available
calibration-cost records are reported in
Appendix~\ref{app:cost}.

\section{Conclusion}
\label{sec:conclusion}

We have formulated the acceleration of a frozen diffusion model as a
state estimation problem: the deep features skipped at a partial
evaluation are an unobserved state, the shallow features computed in
any case are measurements, and a time-dependent affine Koopman
approximation of the increments supplies the dynamics. The resulting
observer corrects its prediction by the innovation of the shallow
measurement, and its error obeys the recursion of a Luenberger
observer in which the identified coupling block plays the role of the
output matrix. In the experiments the correction reduces the error
relative to the full sampler beyond what coupled open-loop prediction
achieves, at speedups of about $1.5$--$1.6\times$ on two
class-conditional U-Nets.

The scope of these results is limited in three respects. The guarantee
of Proposition~\ref{prop:onestep} concerns one step and the
calibration distribution; no contraction of the closed loop is
established, and the growth of the error over an episode is governed
by Eqs.~\eqref{eq:deep_error_recursion}
and~\eqref{eq:deep_feature_accumulation}. The fitted predictor is tied
to the checkpoint, the sampling grid and the conditioning, and its
transfer to another setting requires a check of the prediction
accuracy and possibly a recalibration. Finally, the evaluation covers
two U-Net checkpoints at low resolution; diffusion
transformers, and text-to-image and video models remain to be
examined. A natural
refinement is to fit the gains on the residuals of the closed loop
rather than on the calibration residuals, so that
Proposition~\ref{prop:onestep} applies at every step of an episode.




\subsection*{Acknowledgement}
Y.X. acknowledge support from JST CREST Grant No.~JPMJCR24Q1, including the AIP Challenge Program.

\acknowledgments{Y.X. acknowledge support from JST CREST Grant No.~JPMJCR24Q1, including the AIP Challenge Program.}

\bibliographystyle{plainnat}
\bibliography{references}

@article{bai2026hierarchical,
  title={Hierarchical Koopman diffusion: Fast generation with interpretable diffusion trajectory},
  author={Bai, Hanru and Ding, Weiyang and Zou, Difan},
  journal={Advances in Neural Information Processing Systems},
  volume={38},
  pages={87345--87379},
  year={2026}
}

@inproceedings{bai2025konode,
  title={KoNODE: Koopman-Driven Neural Ordinary Differential Equations with Evolving Parameters for Time Series Analysis.},
  author={Bai, Hanru and Ding, Weiyang}
}

@inproceedings{ho2020denoising,
  title={Denoising diffusion probabilistic models},
  author={Ho, Jonathan and Jain, Ajay and Abbeel, Pieter},
  booktitle={Advances in Neural Information Processing Systems},
  volume={33},
  pages={6840--6851},
  year={2020}
}

@inproceedings{song2020denoising,
  title={Denoising diffusion implicit models},
  author={Song, Jiaming and Meng, Chenlin and Ermon, Stefano},
  booktitle={International Conference on Learning Representations},
  year={2021}
}

@inproceedings{lu2022dpm,
  title={{DPM}-{S}olver: A fast {ODE} solver for diffusion probabilistic model sampling in around 10 steps},
  author={Lu, Cheng and Zhou, Yuhao and Bao, Fan and Chen, Jianfei and Li, Chongxuan and Zhu, Jun},
  booktitle={Advances in Neural Information Processing Systems},
  volume={35},
  pages={5775--5787},
  year={2022}
}

@article{lu2025dpmpp,
  title={{DPM}-{S}olver++: Fast solver for guided sampling of diffusion probabilistic models},
  author={Lu, Cheng and Zhou, Yuhao and Bao, Fan and Chen, Jianfei and Li, Chongxuan and Zhu, Jun},
  journal={Machine Intelligence Research},
  volume={22},
  number={4},
  pages={730--751},
  year={2025},
  doi={10.1007/s11633-025-1562-4}
}

@inproceedings{ma2024deepcache,
  title={{DeepCache}: Accelerating diffusion models for free},
  author={Ma, Xinyin and Fang, Gongfan and Wang, Xinchao},
  booktitle={Proceedings of the IEEE/CVF Conference on Computer Vision and Pattern Recognition (CVPR)},
  pages={15762--15772},
  year={2024}
}

@inproceedings{liu2025reusing,
  title={From reusing to forecasting: Accelerating diffusion models with {TaylorSeers}},
  author={Liu, Jiacheng and Zou, Chang and Lyu, Yuanhuiyi and Chen, Junjie and Zhang, Linfeng},
  booktitle={Proceedings of the IEEE/CVF International Conference on Computer Vision (ICCV)},
  pages={15853--15863},
  year={2025}
}

@article{liu2025fastcache,
  title={{FastCache}: Fast caching for diffusion transformer through learnable linear approximation},
  author={Liu, Dong and Yu, Yanxuan and Zhang, Jiayi and Li, Yifan and Lengerich, Ben and Wu, Ying Nian},
  journal={arXiv preprint arXiv:2505.20353},
  year={2025}
}

@inproceedings{huang2024harmonica,
  title={{HarmoniCa}: Harmonizing training and inference for better feature caching in diffusion transformer acceleration},
  author={Huang, Yushi and Wang, Zining and Gong, Ruihao and Liu, Jing and Zhang, Xinjie and Guo, Jinyang and Liu, Xianglong and Zhang, Jun},
  booktitle={Proceedings of the 42nd International Conference on Machine Learning},
  series={Proceedings of Machine Learning Research},
  volume={267},
  pages={25835--25858},
  year={2025},
  publisher={PMLR}
}

@article{li2026disagree,
  title={Disagree to accelerate: Closing the loop on diffusion feature forecasts},
  author={Li, Yanchao and Xie, Jiaqing and Gao, Ben and Liu, Wanhao and Wang, Yanbo and Tsui, TY and Liu, Jinfei and Li, Yuqiang and Fu, Tianfan},
  journal={arXiv preprint arXiv:2608.01740},
  year={2026}
}

@article{williams2015data,
  title={A data--driven approximation of the {K}oopman operator: Extending dynamic mode decomposition},
  author={Williams, Matthew O. and Kevrekidis, Ioannis G. and Rowley, Clarence W.},
  journal={Journal of Nonlinear Science},
  volume={25},
  number={6},
  pages={1307--1346},
  year={2015}
}

@article{brunton2021modern,
  title={Modern {K}oopman theory for dynamical systems},
  author={Brunton, Steven L. and Budi{\v{s}}i{\'c}, Marko and Kaiser, Eurika and Kutz, J. Nathan},
  journal={SIAM Review},
  volume={64},
  number={2},
  pages={229--340},
  year={2022}
}

@article{luenberger1966observers,
  title={Observers for multivariable systems},
  author={Luenberger, David G.},
  journal={IEEE Transactions on Automatic Control},
  volume={11},
  number={2},
  pages={190--197},
  year={1966}
}

@article{shen2026beyond,
  title={Beyond fixed formulas: Data-driven linear predictor for efficient diffusion models},
  author={Shen, Zhirong and Huang, Rui and Liu, Jiacheng and Zou, Chang and Cai, Peiliang and Zheng, Shikang and Shi, Zhengyi and Feng, Liang and Zhang, Linfeng},
  journal={arXiv preprint arXiv:2604.26365},
  year={2026}
}

@article{han2026adaptive,
  title={Adaptive spectral feature forecasting for diffusion sampling acceleration},
  author={Han, Jiaqi and Shi, Juntong and Li, Puheng and Ye, Haotian and Guo, Qiushan and Ermon, Stefano},
  journal={arXiv preprint arXiv:2603.01623},
  year={2026}
}

@article{chen2026worlddyncache,
  title={{WorldDynCache}: Risk-controlled latent dynamics approximation for diffusion world model},
  author={Chen, Leyang and Wu, Junyi and Zhang, Shaoqiu and Zhang, Yulun},
  journal={arXiv preprint arXiv:2608.01845},
  year={2026}
}

@article{10.1063/5.0283640,
    author = {Xu, Yuanchao and Shao, Kaidi and Ishikawa, Isao and Hashimoto, Yuka and Logothetis, Nikos and Shen, Zhongwei},
    title = {A data-driven framework for Koopman semigroup estimation in stochastic dynamical systems},
    journal = {Chaos: An Interdisciplinary Journal of Nonlinear Science},
    volume = {35},
    number = {10},
    pages = {103123},
    year = {2025},
    month = {10},
    issn = {1054-1500},
    doi = {10.1063/5.0283640},
    url = {https://doi.org/10.1063/5.0283640},
    eprint = {https://pubs.aip.org/aip/cha/article-pdf/doi/10.1063/5.0283640/20757477/103123_1_5.0283640.pdf},
}

\clearpage
\appendix
\thispagestyle{empty}
\onecolumn

\section{Proof of Proposition~\ref{prop:onestep}}
\label{app:proof}

\begin{proof}
For a matrix $L'$ let
$F(L')=\mathbb{E}\|\epsilon^u-L'\epsilon^y\|^2
=\operatorname{tr}\Sigma_{uu}-2\operatorname{tr}(L'\Sigma_{yu})
+\operatorname{tr}(L'\Sigma_{yy}L'^{\top})$.
With $L=\Sigma_{uy}M$ and $M=(\Sigma_{yy}+\gamma I)^{-1}$,
\begin{equation*}
    F(0)-F(L)
    =
    \operatorname{tr}\bigl[\Sigma_{uy}(2M-M\Sigma_{yy}M)\Sigma_{yu}\bigr],
\end{equation*}
\begin{equation*}
    2M-M\Sigma_{yy}M
    =
    M(\Sigma_{yy}+2\gamma I)M,
\end{equation*}
which gives the stated expression. Since $M$ is symmetric and
$\Sigma_{yy}+2\gamma I$ is positive semidefinite,
$F(0)-F(L)=\operatorname{tr}\bigl[(M\Sigma_{yu})^{\top}(\Sigma_{yy}+2\gamma I)(M\Sigma_{yu})\bigr]\ge0$.
For $\gamma=0$, $F$ is a convex quadratic function of $L'$ with
gradient $-2\Sigma_{uy}+2L'\Sigma_{yy}$, which vanishes exactly at
$L=\Sigma_{uy}\Sigma_{yy}^{-1}$; the second moment of
$\epsilon^u-L\epsilon^y$ is
$\Sigma_{uu}-L\Sigma_{yu}-\Sigma_{uy}L^{\top}+L\Sigma_{yy}L^{\top}
=\Sigma_{uu}-\Sigma_{uy}\Sigma_{yy}^{-1}\Sigma_{yu}$.
\end{proof}

Eq.~\eqref{eq:observer_error_identity} follows from
$v_k-\bar v_k=A_{k-1}e_{k-1}+\rho_k$ and
$\hat v_k-\bar v_k=L_kH(v_k-\bar v_k)$;
Eq.~\eqref{eq:deep_error_recursion} is its first block row, with
$e_{k-1}=(e^u_{k-1},0)$;
and Eq.~\eqref{eq:deep_feature_accumulation} follows by summing
$(d_i-\hat d_i)-(d_{i-1}-\hat d_{i-1})
=\sigma_dP_de^u_i+(I-P_dP_d^\top)(d_i-d_{i-1})$
over $i=a+1,\dots,k$.

\section{Additional Experimental Details}
\label{app:experiments}

\subsection{Datasets and Reference Distributions}
\label{app:data}

The CIFAR-10 real-image reference contains the 50,000
training images.
The ImageNet-10 reference contains 13,000 ImageNet
training images, with 1,300 images in each checkpoint
class: airliner, bee, black swan, cup, dingo,
Egyptian cat, electric locomotive, ferret, goldfish,
and panda.
Images are converted to RGB, resized to a shorter side
of 74 pixels with bilinear interpolation, and
center-cropped to $64\times64$.
This is a ten-class subset, not ImageNet-1k evaluation.

Generated floating-point images in $[-1,1]$ are
quantized using
\begin{equation}
    Q(x)
    =
    \operatorname{uint8}
    \left(
        \operatorname{clip}
        \left(
            127.5(x+1)+0.5,\,
            0,\,
            255
        \right)
    \right).
    \label{eq:image_quantization}
\end{equation}

FID, KID, and paired feature errors use the same
2048-dimensional Inception-v3-compatible features
from \texttt{torch-fidelity==0.4.0}.
For each configuration and sampling run, FID is computed
after merging all 10,000 generated feature vectors.
FID scores from separate shards are never averaged.
The reported mean and sample SD are then computed
across the three complete sampling runs.

\subsection{Offline Identification and Data Splits}
\label{app:calibration}

Calibration uses previously collected full-sampling
feature trajectories from the frozen checkpoints (Table \ref{tab:calibration_splits}).
Splits are defined by complete initial-noise trajectory.
All sampled spatial locations and both guidance
branches of a trajectory remain in the same split.
The collector retains 16 fixed spatial locations
per trajectory.

Projection bases, normalization scales, dynamics,
and gains are estimated from the calibration split.
Predictor hyperparameters are selected using the
validation split.
The separate diagnostic split supports development
analyses.
The new 10,000-image sampling runs are separate
from these calibration trajectories.

\begin{table*}[t]
    \centering
    \small
    \caption{
        Complete-trajectory splits for offline identification.
        The diagnostic split is separate from the formal
        generation evaluation.
        Rank is the dimension of each projected increment.
    }
    \label{tab:calibration_splits}
    \begin{tabular}{lrrrr}
        \toprule
        Dataset
        & Calibration
        & Validation
        & Diagnostic
        & Rank $r$ \\
        \midrule
        CIFAR-10
        & 120 & 60 & 60 & 256 \\
        ImageNet-10
        & 80 & 40 & 40 & 256 \\
        \bottomrule
    \end{tabular}
\end{table*}

The main observer has rank 256 for both deep and shallow
increments, giving a joint state dimension of 512.
The selected dynamics ridge parameter is 1.0 in the
implementation's normalized coordinates on both datasets.
The static rank-256 regression uses ridge 0.03.
All identified parameters remain fixed during evaluation.

The calibrated operator models network-feature increments.
It is not identified with the checkpoint's native
bottleneck operator.
No exact finite-dimensional Koopman-invariance claim
is required.
Offline predictor fitting is distinct from training
the denoiser.

\subsection{Refresh Schedules and Baseline Implementations}
\label{app:refresh}

The deterministic sampling grid is
\begin{equation}
    t_k
    =
    \left\lfloor
        999\left(1-\frac{k}{49}\right)
    \right\rfloor,
    \qquad
    k=0,\ldots,49,
\end{equation}
as implemented by integer conversion of a
50-point linear grid.

With maximum partial run length $s$, a step $k$
is fully evaluated when $k<4$, $k\geq48$, or
$(k-2)\bmod(s+2)<2$.
The resulting full/partial counts are 28/22, 23/27,
and 20/30 for $s=2,3,4$.
Each step includes a conditional and an unconditional
evaluation.
These counts are guidance pairs, not single-branch
network-call counts.

Consecutive full evaluations reinitialize feature
increments on the current accelerated trajectory.
They do not reset the image state to a separately
generated reference.

\begin{table*}[t]
    \centering
    \footnotesize
    \setlength{\tabcolsep}{3pt}
    \caption{
        Observer results across the three refresh budgets.
        FID and KID report mean and sample SD across
        three sampling runs.
        Each run generates 10,000 images.
        Speedup is relative to full DDIM-50 at the
        main batch size.
        $E_{\phi}$ reports mean paired feature error.
    }
    \label{tab:observer_budgets}
    \begin{tabular}{lrrrrr}
        \toprule
        Dataset
        & $s$
        & FID$\downarrow$
        & $10^3\mathrm{KID}\downarrow$
        & $10^3E_{\phi}\downarrow$
        & Speedup$\uparrow$ \\
        \midrule
        CIFAR-10
        & 2
        & $13.495\pm0.050$
        & $6.319\pm0.133$
        & 0.743
        & $1.527\times$ \\

        CIFAR-10
        & 3
        & $13.515\pm0.079$
        & $6.224\pm0.098$
        & 1.776
        & $1.733\times$ \\

        CIFAR-10
        & 4
        & $13.669\pm0.071$
        & $6.393\pm0.115$
        & 2.713
        & $1.890\times$ \\

        \midrule
        ImageNet-10
        & 2
        & $13.914\pm0.079$
        & $2.665\pm0.024$
        & 4.662
        & $1.504\times$ \\

        ImageNet-10
        & 3
        & $14.133\pm0.097$
        & $2.760\pm0.032$
        & 9.741
        & $1.701\times$ \\

        ImageNet-10
        & 4
        & $14.257\pm0.074$
        & $2.802\pm0.029$
        & 17.323
        & $1.845\times$ \\
        \bottomrule
    \end{tabular}
\end{table*}

Reduced-step DDIM and DPM-Solver++ use candidate
step counts 20, 24, 26, 28, 30, 32, 34, 36, 38,
and 40.
DPM-Solver++ uses
\texttt{DPMSolverMultistepScheduler} from
\texttt{diffusers==0.38.0}, the checkpoint's beta
schedule, epsilon prediction, solver order 2,
and linear timestep spacing.
Thresholding with maximum sample value 1.0 gives
fixed clipping of the predicted clean image
to $[-1,1]$.
Karras sigma spacing is not used.
These settings define the tested solver configuration
and do not constitute an exhaustive solver search.

Feature reuse, linear extrapolation, and channelwise
affine prediction use the outer decoder feature
location.
Their timing candidates include the $s=4$ schedule
and schedules with 2, 4, 6, 8, 10, 12, or 16
additional full steps.
Additional refreshes prioritize consecutive pairs
to recover exact feature increments.
Equal full-step counts do not imply identical
refresh positions.

DeepCache is adapted to two stage boundaries of the
native U-Net, with uniform refresh intervals of
2, 3, 4, or 6.
It reuses high-level features and recomputes current
shallow decoder skips.
Spectrum uses the saved official predictor code with
previously validation-selected parameters and our
feature location and $s\in\{2,3,4\}$ schedules.
These entries evaluate adaptations on the present
checkpoints, not reproductions of the original
papers' benchmarks.

\paragraph{Latency matching.}
For each dataset and observer budget, we select the
closest-latency configuration within each baseline
family using timing alone.
The middle budget selects DDIM-28 and DPM-Solver++-28
on CIFAR-10, and DDIM-30 and DPM-Solver++-30 on
ImageNet-10.
All seven baseline families are within 3.81\%
of the observer latency at this budget.

At $s=4$, the closest DeepCache candidate is
11.10\% faster on CIFAR-10 and 11.03\% faster
on ImageNet-10.
The closest Spectrum candidate is 11.63\% and
12.60\% slower, respectively.
These four comparisons fall outside the 5\%
tolerance and are not described as latency matched.
Their results remain plotted at their actual
measured latencies.

\subsection{Observation Interventions and Predictor Controls}
\label{app:controls}

The current, delayed, and disabled correction variants
load the same rank-256 base artifact.
They share projection bases, normalization, operators,
gains, and full refresh positions.
All use the unfused implementation.

The delayed variant replaces the correction's current
shallow increment with the preceding sampling step's
measured increment.
Fresh shallow activations are still used by the decoder,
and actual shallow measurements continue to be recorded.
Disabling correction suppresses the joint-state
innovation update, including both its deep and shallow
components.
It does not change only the deep gain while retaining
shallow-state assimilation.

The main rank-256 observer uses an algebraically fused
implementation.
Its full-cohort feature MSE is 0.002712826 on CIFAR-10
and 0.017322962 on ImageNet-10.
The corresponding unfused current-input variant gives
0.002712838 and 0.017323433.
These numerical differences are retained.
Fused and unfused timings are not interchanged in
the strict ablation.

Artifact inspection confirms that the stored open-loop
and base observer files have identical projection bases,
operators, and gains on both datasets.
Their correction flags differ.
The open-loop control therefore reproduces the
disabled-correction outputs in this evaluation.
It is not an independently fitted source of evidence
for the value of correction.

\begin{table*}[t]
    \centering
    \footnotesize
    \setlength{\tabcolsep}{3pt}
    \caption{
        Predictor controls on the same $s=4$ refresh schedule.
        Feature errors report mean and sample SD across
        three sampling runs.
        The rank-256 observer uses fused inference.
        Other configurations can differ in rank, fitting,
        or implementation.
        The strict correction intervention is reported
        in Table~\ref{tab:observation_ablation}.
    }
    \label{tab:dynamics_controls}
    \begin{tabular}{lrr rr}
        \toprule
        & \multicolumn{2}{c}{CIFAR-10}
        & \multicolumn{2}{c}{ImageNet-10} \\
        \cmidrule(lr){2-3}
        \cmidrule(lr){4-5}
        Predictor
        & $10^3E_{\phi}\downarrow$
        & Speedup
        & $10^3E_{\phi}\downarrow$
        & Speedup \\
        \midrule
        Feature reuse
        & $18.270\pm0.144$ & $1.93\times$
        & $50.543\pm0.169$ & $1.89\times$ \\

        Linear extrapolation
        & $4.395\pm0.058$ & $1.92\times$
        & $21.794\pm0.277$ & $1.88\times$ \\

        Channelwise affine
        & $3.388\pm0.067$ & $1.92\times$
        & $19.653\pm0.339$ & $1.88\times$ \\

        Static regression, $r=256$
        & $12.275\pm0.029$ & $1.89\times$
        & $34.979\pm0.361$ & $1.85\times$ \\

        Diagonal observer, $r=64$
        & $3.478\pm0.056$ & $1.90\times$
        & $19.560\pm0.239$ & $1.85\times$ \\

        Dense observer, $r=64$
        & $3.038\pm0.057$ & $1.90\times$
        & $19.460\pm0.198$ & $1.85\times$ \\

        Open-loop, $r=256$
        & $2.842\pm0.050$ & $1.87\times$
        & $18.172\pm0.230$ & $1.83\times$ \\

        Dense observer, $r=256$
        & $2.713\pm0.048$ & $1.89\times$
        & $17.323\pm0.235$ & $1.85\times$ \\
        \bottomrule
    \end{tabular}
\end{table*}

Static regression, rank-64 observation-corrected
dynamics, and rank-64 diagonal dynamics are additional
calibrated configurations.
Their comparisons characterize complete predictors
and may change fitted parameters or numerical
implementation.
In particular, the rank comparison also compares
rank-256 fused inference with rank-64 unfused inference.

We therefore use the same-fit observation interventions
for the most direct test of correction.
These controls do not establish that the observer
is more expressive than every direct affine predictor
with the same input information.

\subsection{Feature Prediction on Reference Trajectories}
\label{app:feature_prediction}

The earlier local diagnostic uses 60 CIFAR-10 and
40 ImageNet-10 development trajectories from full
sampling.
Two consecutive full evaluations initialize the
predictor.
Predictions are assessed at horizons of one, two,
and four steps from anchor indices
3, 9, 15, 21, 27, 33, 39, and 45.

Squared deep-feature errors are aggregated and
normalized by feature reuse at the same anchors
and horizons.
The feature-reuse reference is consequently one.

\begin{table*}[t]
    \centering
    \small
    \caption{
        Local deep-feature prediction on earlier
        development diagnostic trajectories.
        Errors are normalized by feature reuse and
        aggregated over horizons of one, two, and
        four steps.
    }
    \label{tab:feature_diagnostics}
    \begin{tabular}{lrr}
        \toprule
        Predictor
        & CIFAR-10
        & ImageNet-10 \\
        \midrule
        Feature reuse
        & 1.0000 & 1.0000 \\
        Linear extrapolation
        & 0.6022 & 0.6886 \\
        Static shallow-to-deep
        & 0.5964 & 0.5643 \\
        Open-loop dynamics
        & 0.3829 & 0.3907 \\
        Observation-corrected dynamics
        & 0.3664 & 0.3763 \\
        \bottomrule
    \end{tabular}
\end{table*}

Observation correction reduces aggregate diagnostic
error from 0.3829 to 0.3664 on CIFAR-10 and from
0.3907 to 0.3763 on ImageNet-10.
These measurements use shallow observations on full
reference trajectories.
The formal generation results evaluate this
accumulated effect separately.

\subsection{Paired Errors and Uncertainty}
\label{app:paired_metrics}

For accelerated output $\hat{x}_i$ and full DDIM-50
output $x_i^{\mathrm{full}}$ with the same initial
noise and class label, we define
\begin{equation}
    E_{\phi}
    =
    \frac{1}{ND}
    \sum_{i=1}^{N}
    \left\|
        \phi\!\left(Q(\hat{x}_i)\right)
        -
        \phi\!\left(Q(x_i^{\mathrm{full}})\right)
    \right\|_2^2,
    \; D=2048.
    \label{eq:paired_feature_error}
\end{equation}

Pixel MSE $E_x$ is computed directly between
floating-point outputs in $[-1,1]$ and averaged
over pixels, channels, and images.
It is not computed after uint8 quantization.
Full-sampler outputs are generated references,
not ground-truth images.

For each predefined comparison, we jointly resample
paired image indices 1,000 times within each run.
Each bootstrap statistic is the ratio of the two
resampled mean feature errors.
We report the 2.5th and 97.5th percentiles separately
for each run.

These intervals describe variation across images
conditional on the checkpoint and fitted predictor.
They are not FID uncertainty, model-training
uncertainty, or pooled three-run confidence intervals.

\begin{table*}[t]
    \centering
    \footnotesize
    \setlength{\tabcolsep}{3pt}
    \caption{
        Paired feature-error ratios for the strict
        observation interventions at $s=4$.
        Brackets give per-run 95\% paired-bootstrap
        intervals from 1,000 resamples.
        Ratios below one favor current observations.
        Intervals are not pooled across runs.
    }
    \label{tab:observation_bootstrap}
    \begin{tabular}{llrrr}
        \toprule
        Dataset
        & Comparison
        & Run 1
        & Run 2
        & Run 3 \\
        \midrule
        CIFAR-10
        & Current / disabled
        & $0.9529$ & $0.9591$ & $0.9518$ \\
        &
        & $[0.9463,0.9594]$
        & $[0.9521,0.9667]$
        & $[0.9414,0.9616]$ \\

        CIFAR-10
        & Current / delayed
        & $0.8622$ & $0.8774$ & $0.8770$ \\
        &
        & $[0.8500,0.8743]$
        & $[0.8657,0.8887]$
        & $[0.8646,0.8902]$ \\

        \midrule
        ImageNet-10
        & Current / disabled
        & $0.9541$ & $0.9541$ & $0.9516$ \\
        &
        & $[0.9481,0.9602]$
        & $[0.9478,0.9605]$
        & $[0.9456,0.9578]$ \\

        ImageNet-10
        & Current / delayed
        & $0.8260$ & $0.8332$ & $0.8174$ \\
        &
        & $[0.8132,0.8384]$
        & $[0.8203,0.8466]$
        & $[0.8058,0.8289]$ \\
        \bottomrule
    \end{tabular}
\end{table*}

KID is the average of 20 random subset estimates,
each using 500 real and 500 generated feature vectors.
We distinguish the standard deviation across these
subsets from the sample SD of the three run-level
KID means.
Only the latter is reported alongside mean KID
in the result tables.

\paragraph{Classwise analysis.}
After averaging the three runs, current observations
improve on disabled correction in every class of
both datasets.
Relative improvements range from 0.37\% to 9.07\%
on CIFAR-10 and from 0.77\% to 6.51\%
on ImageNet-10.

Relative to delayed correction, current observations
improve nine of ten ImageNet classes.
The airliner class is an exception, with a 5.20\%
increase in error.
Classwise averages are descriptive and do not imply
improvement for every image or every run within
a class.

\subsection{Batch-Size Sensitivity and Offline Cost}
\label{app:cost}

We independently repeat timing at batch sizes
10, 20, 50, and 100 on CIFAR-10, and
5, 10, 20, and 40 on ImageNet-10.
Each configuration is warmed up, and three
synchronized repetitions are measured in randomized
order at each batch size.
The full sampler is timed again to define each
speedup.

\begin{table*}[t]
    \centering
    \footnotesize
    \setlength{\tabcolsep}{4pt}
    \caption{
        Separate batch-size timing measurements for $s=4$.
        Latencies are seconds per batch, measured as the
        median of three synchronized repetitions after
        warmup.
        Each speedup uses the full sampler timed again
        at the corresponding batch size.
    }
    \label{tab:batch_sensitivity}
    \begin{tabular}{lrrrrrr}
        \toprule
        Dataset
        & Batch
        & Full
        & Reuse
        & Channel
        & Observer
        & Speedup \\
        \midrule
        CIFAR-10
        & 10 & 3.399 & 1.586 & 1.589 & 1.603
        & $2.121\times$ \\
        CIFAR-10
        & 20 & 4.166 & 2.119 & 2.124 & 2.149
        & $1.939\times$ \\
        CIFAR-10
        & 50 & 7.043 & 3.657 & 3.671 & 3.728
        & $1.889\times$ \\
        CIFAR-10
        & 100 & 12.507 & 6.565 & 6.594 & 6.706
        & $1.865\times$ \\

        \midrule
        ImageNet-10
        & 5 & 4.093 & 2.098 & 2.102 & 2.130
        & $1.922\times$ \\
        ImageNet-10
        & 10 & 5.959 & 3.069 & 3.082 & 3.131
        & $1.903\times$ \\
        ImageNet-10
        & 20 & 9.717 & 5.150 & 5.173 & 5.265
        & $1.846\times$ \\
        ImageNet-10
        & 40 & 17.315 & 9.413 & 9.459 & 9.637
        & $1.797\times$ \\
        \bottomrule
    \end{tabular}
\end{table*}

For $s=4$, observer speedups range from
$1.87\times$ to $2.12\times$ on CIFAR-10
and from $1.80\times$ to $1.92\times$
on ImageNet-10.
These are throughput-dependent timing results.
The formal quality evaluation uses only the
main batch sizes 50 and 20.

The completed GPU suite reuses previously calibrated
predictors and does not measure end-to-end offline
identification cost.
Earlier collector logs record trajectory collection
but do not provide a complete, hardware-matched
measurement including basis construction, operator
fitting, and gain fitting.

The serialized base observer occupies approximately
73.37 MiB per dataset.
The fused deployment artifact occupies approximately
36.05 MiB, excluding the denoiser.
Serialized file size is not peak accelerator
memory usage.

\subsection{Evaluation Records and Reproducibility}
\label{app:full_results}

The timing pool contains 66 configurations per dataset,
including the proposed method and ablations.
The expanded control set and timing-selected baselines
yield 34 distinct evaluated configurations per dataset.

Across two datasets, three sampling runs, and
10,000 images per run, the evaluation completes
2,040,000 configuration-level generations.
Initial noise is shared across methods, so this
count does not represent 2,040,000 independent
initial conditions.

All configurations are retained in the evaluation
records, including those that do not improve FID
or paired fidelity.
The records contain per-run FID and KID,
per-image paired errors, classwise summaries,
timing repetitions, configuration hashes,
and completion markers.
All 204 configuration-level runs contain
10,000 samples.

Recorded correctness checks compare the split
full-forward path with native DDIM and compare
each DeepCache full-forward branch with native
execution.
Maximum output differences in these checks
are zero on both datasets.
The current-observation intervention also matches
the existing unfused base observer exactly on
the check inputs.
These checks validate the tested execution paths.

\subsection{Qualitative Evaluation and Scope}
\label{app:scope}

For the formal evaluation, preview grids retain
the first 40 image indices of each run,
corresponding to four images per class.
The same indices are used across methods.
These previews are not selected according to
improvement.

The selected ImageNet teaser examples in the
introduction come from an earlier development
cohort and illustrate individual improvement
cases.
They are separate from the fixed-index
formal-evaluation previews and should not be
interpreted as representative random samples.

\end{document}